\documentclass[letterpaper]{article} 
\usepackage{aaai23}  
\usepackage{times}  
\usepackage{helvet}  
\usepackage{courier}  
\usepackage[hyphens]{url}  
\usepackage{graphicx} 
\usepackage{natbib}  
\usepackage{caption} 
\usepackage{algorithm}
\usepackage{algorithmic}
\usepackage{multirow}
\usepackage{booktabs}
\usepackage{bm} 
\newtheorem{theorem}{Theorem}
\newtheorem{lemma}{Lemma}
\usepackage{array}
\usepackage{enumitem}
\usepackage{graphicx}
\usepackage{amsmath}
\usepackage{amsfonts}
\usepackage{tablefootnote}
\usepackage{xcolor}
\newcommand{\eat}[1]{}
\usepackage{newfloat}
\usepackage{listings}
\DeclareCaptionStyle{ruled}{labelfont=normalfont,labelsep=colon,strut=off} 
\floatstyle{ruled}
\newfloat{listing}{tb}{lst}{}
\floatname{listing}{Listing}
\title{Human Mobility Modeling During the COVID-19 Pandemic \\via Deep Graph Diffusion Infomax}
\author {
    Yang Liu\textsuperscript{\rm 2},
    Yu Rong\textsuperscript{\rm 3},
    Zhuoning Guo\textsuperscript{\rm 1},
    Nuo Chen\textsuperscript{\rm 1},
    Tingyang Xu\textsuperscript{\rm 3},
    Fugee Tsung\textsuperscript{\rm 1, \rm 2},
    Jia Li\textsuperscript{\rm 1, \rm 2}\thanks{Corresponding author.}
}
\affiliations {
    \textsuperscript{\rm 1}The Hong Kong University of Science and Technology (Guangzhou)\\
    \textsuperscript{\rm 2}The Hong Kong University of Science and Technology\\
    \textsuperscript{\rm 3}Tencent AI Lab\\
    yliukj@connect.ust.hk, yu.rong@hotmail.com,\{zguo772,nchen022\}@connect.hkust-gz.edu.cn, tingyangxu@tencent.com,\{season,jialee\}@ust.hk
}

\usepackage{bibentry}
\begin{document}

\maketitle

\begin{abstract}
Non-Pharmaceutical Interventions (NPIs), such as social gathering restrictions, have shown effectiveness to slow the transmission of COVID-19 by reducing the contact of people. To support policy-makers, multiple studies have first modeled human mobility via macro indicators (e.g., average daily travel distance) and then studied the effectiveness of NPIs. In this work, we focus on mobility modeling and, from a micro perspective, aim to predict locations that will be visited by COVID-19 cases. Since NPIs generally cause economic and societal loss, such a micro perspective prediction benefits governments when they design and evaluate them. However, in real-world situations, strict privacy data protection regulations result in severe data sparsity problems (i.e., limited case and location information).
To address these challenges, we formulate the micro perspective mobility modeling into
computing the relevance score between a diffusion and a location, conditional on a geometric graph. we propose a model named \textit{Deep Graph Diffusion Infomax} (DGDI), which jointly models variables including a geometric graph, a set of diffusions and a set of locations.
To facilitate the research of COVID-19 prediction, we present two benchmarks that contain geometric graphs and location histories of COVID-19 cases. Extensive experiments on the two benchmarks show that DGDI significantly outperforms other competing methods.
\end{abstract}

\section{Introduction}
The COVID-19 pandemic has been the greatest global public health challenge with over 567 million confirmed cases and over 6.3 million deaths as of 27 July 2022\footnote{\url{who.int/publications/m/item/weekly-epidemiological-update-on-covid-19---27-july-2022}}. The outbreak of the COVID-19 not only threatens the public health but also has a devastating impact on economic activity, leading to increased food insecurity, poverty, and socioeconomic inequality~\cite{wang2020addressing,laborde2020covid,clouston2021socioeconomic}. To control the spread of COVID-19, Non-Pharmaceutical Interventions (NPIs), such as social gathering restrictions, school and business closures, and stay-at-home orders, are implemented via changing people's mobility behaviours~\cite{haug2020ranking}. Therefore, multiple studies~\cite{levin2021insights,hu2021big,gibbs2020changing,gibbs2021detecting} have analyzed the human mobility patterns to support policy-makers. They usually aggregate mobility data, provided by commercial companies, such as SafeGraph~\footnote{https://www.safegraph.com/academics} and Google~\footnote{https://www.google.com/covid19/mobility/}, and then compute metrics (e.g., daily average fraction of residents staying at home) to model human mobility. Despite their effectiveness, they only focus on a macro aspect of mobility during COVID-19 and ignore the impact of NPIs on individuals.
Since NPIs slow the virus spread by limiting human mobility, an in-depth understanding of human behaviours during the COVID-19 pandemic will allow decision makers judiciously and promptly implement interventions.

In this work, we strictly follow the real-world scenario that case diffusions (i.e., location visiting histories) can only be observed when they are tested positive, which abides by the privacy data protection regulation of most countries/regions. Under these settings, we aim to answer the question: \emph{will a location be visited by the COVID-19 cases (e.g., confirmed cases, close contacts, and asymptomatic infections) in the near future?} 
Addressing the above problem is challenging due to the \textbf{data sparsity} in two aspects. 

Firstly, due to the privacy policy, the available diffusion data is anonymous, resulting in a limited number of location visiting records of each case. To alleviate such data sparsity, a common solution is to incorporate side information. Thus,
we propose to utilize the geometric information, i.e., a geometric graph is constructed where two locations are connected if they are neighbors. To this end, our problem is formulated into computing the relevance score between a diffusion and a location, conditional on a geometric graph. Following the most recent works~\cite{DBLP:conf/kdd/ChangWLMDPKGRGM21,2020Mobility,DBLP:conf/kdd/SchwabePF21}, we employ the representation learning framework. That is, we use fixed-length embedding vectors to denote locations, diffusions, and graphs respectively. Figure~\ref{fig:example} is an example of the COVID-19 high-risk location prediction, in which a diffusion illustrates a location history of the COVID-19 case, and the geometric graph reflects the distance between locations. 

Secondly, most locations have a small number of appearances. Generally, in real-world diffusions, common locations (head nodes) can have sufficient appearances while some uncommon locations (tail nodes) can be underrepresented by limited appearances~\cite{anderson2006long}. This imbalance poses great challenge to learning unbiased representations, making the learned representations easily dominated by head nodes. To address this limitation, we extend the idea of Mutual Information (MI) maximization ~\cite{DBLP:journals/corr/abs-1807-03748} to the COVID-19 mobility prediction, which has been proven to be effective in learning high-quality representations with skew distributions~\cite{wang2021contrastive}.
\begin{figure*}
\begin{center}
\includegraphics [width=0.95\textwidth]{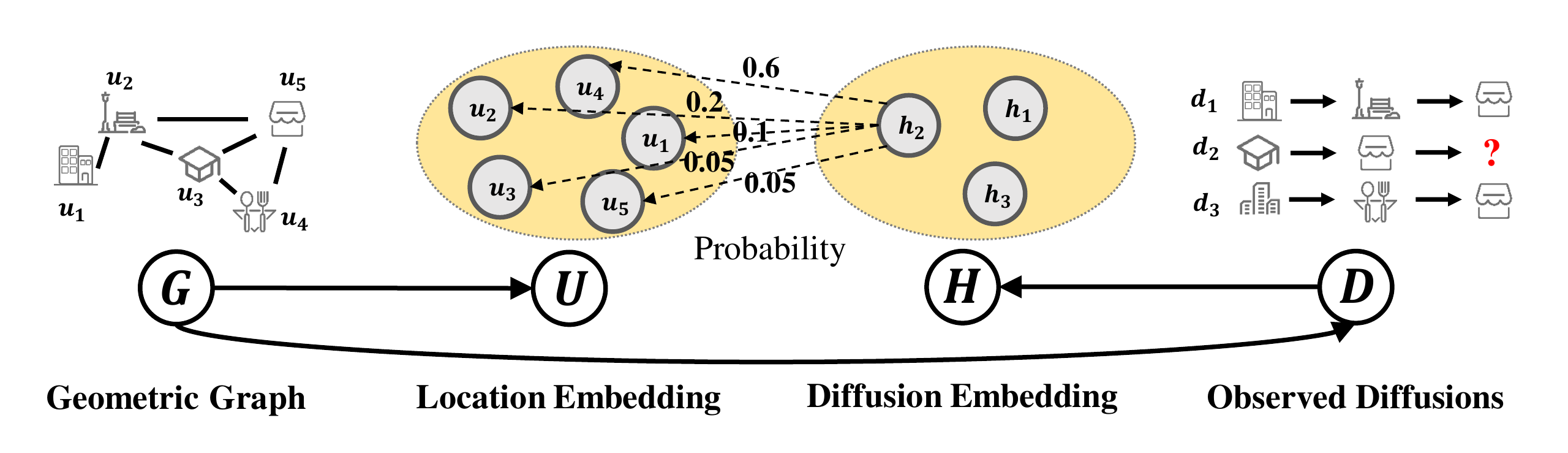}
\end{center}
\caption{An example of the COVID-19 mobility prediction. The geometric graph records the distance information (two points are connected if they are within a Euclidean distance). The diffusion illustrates a location history of the COVID-19 case.}
\label{fig:example}
\end{figure*}
Unlike previous graph MI maximization methods~\cite{DBLP:conf/iclr/VelickovicFHLBH19} which only have single input (i.e., graph) and output (i.e., location embedding), our formulation involves multiple variables including two inputs (i.e., the geometric graph and diffusions) and two representations (i.e., location and diffusion representations). It is difficult to enforce consistency among these four variables within the framework of MI maximization. In this work, we present \textbf{Deep Graph Diffusion Infomax (DGDI)}, to tackle the above challenges. DGDI is defined between two joint distributions: the joint distribution of the graph input and location representations and that of diffusions and diffusion representations. \eat{Intuitively, optimizing DGDI allows the model to explicitly consider the dependence between graph and cascade sub-systems.} Moreover, our theoretical derivations show that DGDI can be lower bounded by a linear combination of two univariate MI: the univariate MI between the graph input and diffusion representations and the univariate MI between diffusion representations and location representations. Upon this decomposition, the univariate MI can be easily computed via InfoNCE~\cite{DBLP:journals/corr/abs-1807-03748} or Deep Graph Infomax (DGI)~\cite{DBLP:conf/iclr/VelickovicFHLBH19}. We then derive the visit probability by a similarity function of diffusion and location representations.

\paragraph{\textbf{Social impact}} Given the existence of immune escape (i.e., vaccines may fail to protect people)~\cite{zhang2022significant}, especially for some COVID-19 variants (e.g., Omicron), and the specific drug for COVID-19 is still in an early stage, non-pharmaceutical interventions are still practical and effective to combat COVID-19. However, since these interventions may cause substantial economic and societal loss while hurt individuals' mental health and social security, effectively quantifying human mobility during the COVID-19 pandemic to assess NPI effectiveness is necessary to balance costs. Besides, human mobility models can benefit the control of other epidemics as well, such as influenza~\cite{venkatramanan2021forecasting} and ebola~\cite{peak2018population}.

\section{Related Work}
In this section, we review human mobility and graph contrastive learning papers that are related to our work.

\subsection{\textbf{Human Mobility Modeling}}
Non-pharmaceutical interventions have been one of the most effective tools to defeat COVID-19 transmission and most NPIs aim to change human mobility to reduce the contact rate.
Outside the Artificial Intelligence (AI) community, plenty of works~\cite{gibbs2020changing,hu2021big,levin2021insights,gibbs2021detecting,lai2020effect,chang2021mobility} have studied the relations between human mobility and NPIs or COVID-19 transmission, which could provide valuable insights for future public health efforts. Most of them focus on computing mobility metrics via aggregating mobility data, e.g., daily average fraction of residents staying at home~\cite{levin2021insights}, daily average number of trips~\cite{hu2021big} and movement flow matrix~\cite{gibbs2020changing}. Nevertheless, these macro-view researches are too simple to model real-world mobility, as they treat individuals without differentiation. Note that in the AI domain, there are analogous tasks such as mobility prediction and diffusion prediction which can be utilized to model individual trajectories. Thus in this work, we fill this gap and derive a mobility prediction method for COVID-19 cases. We believe such an approach can help discover more valuable insights during NPI designs.

Existing mobility prediction methods in AI community can be broadly categorized into: Matrix Factorization (MF)~\cite{DBLP:conf/kdd/LiuFYX13}, Markov Chain (MC)~\cite{DBLP:conf/gis/ZhangCL14,DBLP:conf/pakdd/ChenLY14}, and deep learning models which consist of recurrent models~\cite{DBLP:conf/cikm/YaoZHB17,DBLP:conf/aaai/SunQCLNY20,DBLP:conf/ijcai/LiHWY21}, attention mechanism~\cite{DBLP:conf/www/YuCGLLL20}, and graph neural networks~\cite{DBLP:conf/cikm/LimHNWGWV20}. Diffusion prediction models follow a similar research line, from early simple linear models~\cite{DBLP:conf/kdd/KempeKT03,granovetter1978threshold,DBLP:conf/icdm/BarbieriBM12,DBLP:conf/acml/SaitoKOM09,rong2016model,rong2015happened} to recent deep learning based models~\cite{DBLP:conf/icde/FengCKLLC18,DBLP:conf/icdm/WangZLC17,DBLP:conf/wsdm/SankarZK020,DBLP:conf/sigir/00040LSL0WSYA21,DBLP:conf/ijcai/WangSLGC17,DBLP:conf/cikm/WangCL18,DBLP:conf/icdm/IslamMAPR18,DBLP:journals/tkde/YangSLHLL21} including recurrent neural networks (RNNs) and variational autoencoders (VAEs). 

Compared to diffusion prediction methods, mobility prediction models usually have an additional user representation model learned from users' history trajectories (users generally have more than one trajectory) and user features. Nevertheless, in our task, users have no such rich features but a single trajectory due to privacy protection. In this regard, we follow the settings of diffusion prediction which only rely on the location visiting history and a graph to make predictions.

\subsection{\textbf{Graph Contrastive Learning}}
The recent advance of contrastive learning in computer vision motivates the studies of graph contrastive learning~\cite{DBLP:conf/iclr/VelickovicFHLBH19,DBLP:conf/www/JingPT21,DBLP:conf/www/PengHLZRXH20,DBLP:conf/nips/YouCSCWS20,DBLP:conf/iclr/SunHV020,DBLP:conf/kdd/QiuCDZYDWT20}. DGI~\cite{DBLP:conf/iclr/VelickovicFHLBH19} maximizes the mutual information between node representations and global summary of the graph. HDMI~\cite{DBLP:conf/www/JingPT21} leverages two signals to train the model: the mutual information between location embedding and global summary, and the mutual dependence between location embedding and location attributes. HGMI~\cite{li2022semi} maximizes the mutual information among three variables. To the best of our knowledge, existing work~\cite{DBLP:conf/kdd/0001S21} only explores the advantage of self-supervised learning on COVID-19 Cough Classification, aiming at learning robust representations of respiratory sounds. Our work is the first study that models the mutual information among the graph, nodes and diffusions in COVID-19 diffusion prediction domain. 

\section{Preliminaries}
\subsection{Problem Definition}\label{pdd}
A set of nodes $\bm{V}=\{v_1, v_2, \ldots, v_N\}$ is used to represent real-world locations or areas in COVID-19 transmission. In this work, locations and nodes are used interchangeably. An $N\times N$ adjacency matrix $\bm{A}$ is used to describe a \textbf{geometric graph} $\bm{G}=\{\bm{V},\bm{A}\}$ for nodes in $\bm{V}$.  $\bm{A}_{ij}\in \{0, 1\}$ represents whether there is an edge between nodes $v_i$ and $v_j$ or not, e.g., nodes connected by nearest neighbors in COVID-19 transmission.  
Let a set of \textbf{diffusions} $\bm{D}=\{d_1, d_2, \ldots, d_M\}$ represent the COVID-19 cases, in which a diffusion $d_i$ is an ordered sequence of location visit histories in ascending order of time denoted by $d_i = \{v_{i_1}, v_{i_2}, \ldots, v_{i_k}\}$. The $k^{th}$ visited location of $d_i$ is recorded as $v_{i_k}$.

Given the pair $\{\bm{G},\bm{D}\}$, our problem is defined as learning a prediction model $f(\bm{G},\bm{D})$ to estimate the probability of visiting an unvisited location $v$: $p(v|d_i)$ $\exists\ v \in \bm{V} - \{v_i\}_{k=1}^K$\eat{, where $d_i$ is a known activated location sequence $d_i = \{v_{i_1}, \ldots, v_{i_k}\}$}. 
Due to a large number of potential unvisited locations, we formulate our prediction as an information retrieval problem. For example, for $d_1$ in Figure~\ref{fig:example}, given the observed visited locations (i.e., hotel, park, and mall), the model ranks the probability of visiting unvisited locations (i.e., school, canteen, flat).

\subsection{Framework}
In our problem setting, we have two information sources: a geometric graph $\bm{G}$ and a set of diffusions $\bm{D}$. Our purpose is to learn diffusion representations $\bm{H}=\{h_i\}_{i=1}^M$ and location representations $\bm{U}=\{u_i\}_{i=1}^N$, where $h_i$ denotes the derived diffusion representation for $d_i$ and $u_i$ denotes the derived location representation for location $v_i$. These representations can be retrieved and used to compute a visit likelihood $p(v|d_i)$.

We seek to maximize mutual information for two subsystems: the graph joint distribution $(\bm{G},\bm{U})$ to model the geometric information and the diffusion joint distribution $(\bm{D},\bm{H})$ to model the COVID-19 cases, namely,
\begin{equation}
\arg\max  I((\bm{G},\bm{U});(\bm{D},\bm{H})),
\label{equ.unsuper}
\end{equation}
Intuitively, this target enforce consistency between the geometric information and the COVID-19 cases, which is named \textbf{Deep Graph Diffusion Infomax (DGDI)}.

We consider the geometric graph $\bm{G}$, the location representations $\bm{U}$, the diffusions $\bm{D}$ and the diffusion representations $\bm{H}$ as four random variables and forming a probabilistic graphical model, as illustrated in Figure \ref{fig:example}. As the figure makes clear, the diffusion $d_i$ is assumed to be sampled once from the geometric graph $\bm{G}$; the location representation $u_i$ directly depends on the graph $\bm{G}$; the diffusion representation $\bm{H}$ directly depends on the diffusion $\bm{D}$. As the joint distribution in Eq. \ref{equ.unsuper} makes the data appearances more sparse, we seek to find a marginal distribution lower bound to relieve the data sparsity issue.

Based on the above assumptions, we have the following decomposition theorem to make the computation of DGDI feasible.
\begin{theorem}
    \emph{The graph diffusion mutual information $I((\bm{G},\bm{U});(\bm{D},\bm{H}))$  can be lower bounded by  a sum of univariate mutual information, namely,
    \begin{equation}
    I((\bm{G},\bm{U});(\bm{D},\bm{H})) \geq \frac{1}{4}I(\bm{G};\bm{H}) + \frac{1}{4}I(\bm{H};\bm{U})
    \end{equation}
    here $I(\bm{G};\bm{H})$ is the mutual information between graph input and diffusion representations, $I(\bm{H};\bm{U})$ is the mutual information between diffusion representations and location representations.
}
    \label{thm:submod}
\end{theorem}

 To prove Theorem \ref{thm:submod}, we first introduce one lemma.
 \begin{lemma}
    \emph{For an arbitrary set of three random variables $R_1$, $R_2$, $R_3$, we have
    \begin{equation}
    I(R_1;R_2,R_3) \geq \frac{1}{2}(I(R_1;R_2) + I(R_1;R_3))
    \end{equation}
    here $I(R_1;R_2,R_3)$ is the mutual information between variable $R_1$ and the joint distribution of $R_2$ and $R_3$ .
}
    \label{thm:l1}
\end{lemma}
 To prove Lemma \ref{thm:l1}, we make use of the chain rule for mutual information.
\begin{align}
\begin{split}
   I(R_1;R_2,R_3) &= I(R_1;R_3) + I(R_1;R_2|R_3)\\
   &\geq I(R_1;R_3)
\label{geq6}
\end{split}
\end{align}
The last inequality holds as mutual information is non-negative. Accordingly, we have
\begin{equation}
    I(R_1;R_2,R_3) \geq I(R_1;R_2)
    \label{geq7}
 \end{equation}
Based on Eq. \ref{geq6} and Eq. \ref{geq7}, we complete the proof of Lemma \ref{thm:l1}. 
\eat{
 \begin{lemma}
    \emph{With the Markov assumptions, we have
    \begin{equation}
    I((\bm{G},\bm{U});(\bm{D},\bm{H})) \geq \frac{1}{2}(I(\bm{G};\bm{D}) + I(\bm{G};\bm{H}))
    \end{equation}
}
    \label{thm:l2}
\end{lemma}
\begin{proof}
\begin{align}
\begin{split}
    I((\bm{G},\bm{U});(\bm{D},\bm{H})) &= I(\bm{G};(\bm{D},\bm{H})) + I(\bm{U};(\bm{D},\bm{H})|\bm{G})\\
   &= I(\bm{G};(\bm{D},\bm{H})) + 0\\
   &\geq \frac{1}{2}(I(\bm{G};\bm{D}) + I(\bm{G};\bm{H}))
\label{geq8}
\end{split}
\end{align}
The first equality holds because of the chain rule for mutual information; $I(U;(D,H)|G) = 0$ holds as two sets of nodes are conditional independent if there is no un-blocked path between two sets, or $U$ and $(D,H)$ in our case; the last inequality holds according to Lemma \ref{thm:l1}. Thus, we complete the proof of Lemma \ref{thm:l2}. 
\end{proof}
 \begin{lemma}
    \emph{With the Markov assumption, we have
    \begin{equation}
    I((\bm{G},\bm{U});(\bm{D},\bm{H})) \geq \frac{1}{4}(I(\bm{H};\bm{U}) + I(\bm{G};\bm{H})) + \frac{1}{2}I(\bm{G};\bm{D})
    \end{equation}
}
    \label{thm:l3}
\end{lemma}}

We then prove Theorem \ref{thm:submod},
\begin{align}
\begin{split}
   &I((\bm{G},\bm{U});(\bm{D},\bm{H})) \geq \frac{1}{2}I((\bm{G},\bm{U});\bm{H}) + \frac{1}{2}I((\bm{G},\bm{U});\bm{D})\\
   &\geq \frac{1}{4}I(\bm{U};\bm{H}) + \frac{1}{4}I(\bm{G};\bm{H}) + \frac{1}{2}I((\bm{G},\bm{U});\bm{D})\\
   &= \frac{1}{4}I(\bm{U};\bm{H}) + \frac{1}{4}I(\bm{G};\bm{H}) + \frac{1}{2}I(\bm{G};\bm{D}) + \frac{1}{2}I(\bm{U};\bm{D}|\bm{G})\\
   &= \frac{1}{4}I(\bm{U};\bm{H}) + \frac{1}{4}I(\bm{G};\bm{H}) + \frac{1}{2}I(\bm{G};\bm{D}) + 0\\
   &\geq \frac{1}{4}I(\bm{U};\bm{H}) + \frac{1}{4}I(\bm{G};\bm{H})
\label{geq9}
\end{split}
\end{align}
The first and second inequalities hold according to Lemma \ref{thm:l1}; the third equality holds because of the chain rule for mutual information; $I(U;D|G) = 0$ holds as $U$ and $D$ are conditional independent given $G$ (under the proposed graphical model); the last inequality holds as $I(\bm{G};\bm{D})$ is positive and irrelevant with our optimization variables. Thus, we complete the proof of Theorem \ref{thm:submod}. 

\section{Model}
In this section, we first introduce our method to compute the location, diffusion and graph representations, then describe how to maximize DGDI in details.

\subsection{Representation Model}
The goal of our representation model is to produce fixed-length embedding vectors of locations, diffusions and graph. We first obtain location embeddings $\bm{U}$ from random initializing and then propose a simple yet effective way to compute graph and diffusion representations.

\subsubsection{\textbf{Diffusion representation}}
 To capture the graph characteristics and the diffusion temporal influence, the diffusion representation part takes the geometric graph $G$ and a set of diffusions $D$ as inputs. Accordingly, our model consists of two major components: (1) a Graph Neural Network (GNN) that smooths each location representation according to the graph topology; and (2) a self-attention layer that quantifies the varying effect of previous visited locations.
Our objective is to encode the graph information into location embeddings in GNN and then feed them into the temporal self-attention layer. 

We use GNN model to encourage locations that are close in the geometric graph $G$ to share similar latent representations, which may benefit the predicted task in the absence of explicit location attributes. For example, in the COVID-19 spread, locations that are close in geometric space may be visited by the same cases. We empirically test Graph Convolutional Network (GCN)~\cite{DBLP:conf/iclr/KipfW17}, Graph Attention Network (GAT)~\cite{DBLP:conf/iclr/VelickovicCCRLB18}, and Graph Isomorphism Network (GIN), and find that GCN performs best. Its update rule is: 
\begin{align}\label{eq:gcn}
\bm{Z}^{(l)} = \text{ReLU}(\hat{\bm{A}}\bm{Z}^{(l-1)}\bm{W}^{(l-1)}),
\end{align}
where $\hat{\bm{A}}$ is a normalized adjacency matrix and $\bm{Z}^{(l)}$ ($\bm{Z}^{(0)} = \bm{U}$) denotes the output representations of $l$-th layer GCN.

The latent embedding $\bm{Z}$ is then utilized to compute the diffusion embedding. In the COVID-19 predictions, one observation is that the true spread processes may not strictly follow the sequential assumption, e.g., a location can be visited depending on any of the previous visited locations. To capture such dependency, a self-attention layer followed by a Multi-Layer Perceptron (MLP) is adopted. Specifically, given a diffusion $d_i = \{v_{i_1}, v_{i_2}, \ldots, v_{i_k}\}$, the $k$-th location is represented by $z'_{k}$ (we omit the subscript $i$ for simplicity):
\begin{align}
z'_{k} = z_{k} + \text{PE}(k)
\end{align}
where $\text{PE}(k)$ is the positional-encoding~\cite{DBLP:conf/nips/VaswaniSPUJGKP17} that only depends on the position $k$. The even and odd elements are $\text{sin}(k/10000^{i/L})$ and $\text{cos}(k/10000^{i-1/L})$ respectively ($L$ is the dimension of encoding). Then the diffusion representation $h_i$ for the diffusion $d_i$ can be computed by:
\begin{align}
h_{i} &= \text{MLP}\left(\sum_{j=1}^{k}\alpha_{jk}f_v(z'_{j})\right),\label{eq:atn1}\\
\alpha_{jk} &= \frac{\text{exp}(<f_q (z'_{k}), f_k (z'_{j})>)}{\sum_{j=1}^{k}\text{exp}(<f_q (z'_{k}), f_k (z'_{j})>)},\label{eq:atn2}
\end{align}
where $<\cdot,\cdot>$ denotes the inner product of two vectors. $z'_{k}$ denotes the last location of $d_i$ and attention weight $\alpha_{jk}\in\mathbb{R}$ denotes the $j$-th node's contribution to the final diffusion representation. Following the previous work~\cite{DBLP:conf/nips/VaswaniSPUJGKP17}, we apply three linear transform functions $f_k$, $f_q$, and $f_v$ (i.e., $f(z) = z\bm{W}$) on latent embeddings.

\subsubsection{\textbf{Graph representation}}\label{gr_m}
The representation of the entire graph can be obtained by aggregating the output of GNN, which is termed as graph pooling~\cite{DBLP:conf/nips/YingY0RHL18}. The pooling function transforms arbitrary-sized location embeddings to a fixed-length representation, which can be a simple mean or a more sophisticated graph-level pooling function such as Attention Pooling~\cite{DBLP:conf/www/LiRCMHH19,DBLP:conf/icml/LeeLK19}. Here we use mean pooling for simplicity:
\begin{align}
g = \frac{1}{N}\sum_{i=1}^{N}z_i,
\end{align}
where $g$ denotes the graph representation.

\subsection{Maximization of DGDI}
\begin{figure}[t]
\begin{center}
\includegraphics [width=0.46\textwidth]{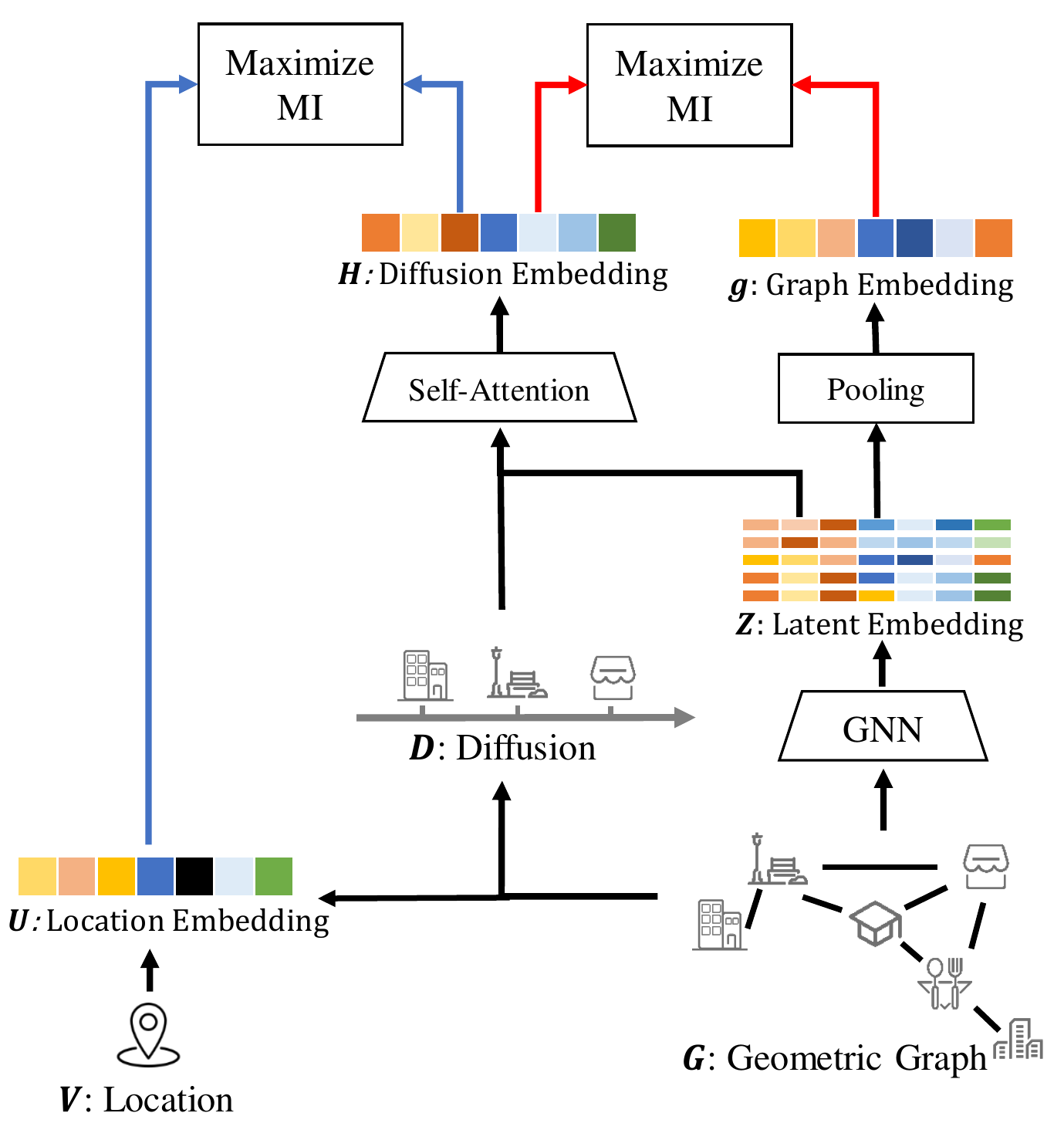}
\end{center}
\caption{Overview of the proposed DGDI. The graph embedding is obtained via pooling the latent embeddings, and the diffusion embedding is calculated via self-attention with the input of location embeddings. Finally, we compute the MI between the location embedding and diffusion embedding, and the MI between the diffusion and graph embedding.}
\label{fig:gcmi}
\end{figure} 

\paragraph{\textbf{Training objective}} 
According to Theorem~\ref{thm:submod}, DGDI is lower bounded by two univariate mutual information: mutual information between graph input and diffusion representations and that between diffusion representations and location representations. 
Accordingly, our training objective becomes:
\begin{align}
\mathcal{L} = \lambda_1 \mathcal{L}_{G} + \lambda_2 \mathcal{L}_{U}
\end{align}
Following existing work~\cite{DBLP:conf/www/PengHLZRXH20}, $\lambda_1$ and $\lambda_2$ are set to tunable parameters for better performance. $\mathcal{L}_{G}$ and $\mathcal{L}_{U}$ are the training losses for $I(\bm{G};\bm{H})$ and $I(\bm{H};\bm{U})$ maximization respectively. Figure~\ref{fig:gcmi} illustrates the overall framework of DGDI. 
During our experiments, we found that deep graph infomax (DGI)~\cite{DBLP:conf/iclr/VelickovicFHLBH19} and InfoNCE~\cite{DBLP:journals/corr/abs-1807-03748} are simple but effective methods to estimate $I(\bm{G};\bm{H})$ and $I(\bm{H};\bm{U})$. They share a similar idea which encourages the consistency between the representations of positive pairs and the divergence between that of negative pairs. Although there exist other estimators (e.g., GMI~\cite{DBLP:conf/www/PengHLZRXH20} or InfoGraph~\cite{DBLP:conf/iclr/SunHV020}), we empirically show that DGI and InfoNCE have already achieve state-of-the-art performance in Section~\ref{sec:exp}. Thus, we leave the testing of other choices for future work.

\paragraph{\textbf{$I(\bm{G};\bm{H})$ computation}}
MI between global graph embedding $g$ and representations of local parts are maximized under the framework of DGI~\cite{DBLP:conf/iclr/VelickovicFHLBH19}. Since diffusions are assumed to be sampled from the graph, we treat diffusions as the local parts and maximize MI between global graph embedding $g$ and diffusion embedding $h$. A negative sampling strategy is then leveraged to generate negative graph $\hat{g}$. More specifically, we randomly shuffle the location embedding matrix and get $\hat{\bm{U}}$; we let $\hat{\bm{U}}$ go through the GNN module and  the graph representation module and get $\hat{g}$. 
A discriminator is then utilized to distinguish the diffusion representation and graph representations. Formally, the loss is defined as follows:
\begin{align}
\mathcal{L}_G = \sum_{h\in \bm{H}}\mathbb{E}[\text{log}\ f_d(h, g)] + \mathbb{E}[\text{log}\ (1 - f_d(h, \hat{g}))],
\end{align}
where $f_d$ denotes the inner-product discriminator:
\begin{align}
f_d(h, g) = \sigma (<h, g>),
\end{align}
where $\sigma$ denotes the sigmoid activation function and $f_d(h, \hat{g})$ is calculated in the same way.

\paragraph{\textbf{$I(\bm{H};\bm{U})$ computation}}
$I(\bm{H};\bm{U})$ is maximized between diffusion representation $H$ and location representation $U$. We use InfoNCE~\cite{DBLP:journals/corr/abs-1807-03748} to maximize the MI of these two representations. In terms of sampling strategies, a diffusion and its next visited location are considered as a positive pair while a diffusion and uninfected locations are as negative pairs. To summarize, $\mathcal{L}_{U}$ is computed by:
\begin{align}
\mathcal{L}_{U} =-\sum_{i=1}^{M}\sum_{k=1}^{K}\text{log} \frac{\text{exp}(<h_{i_k}, u> / \tau)}{\sum_{u'\in \bm{U}}\text{exp}(<h_{i_k}, u'> / \tau)},
\end{align}
where $h_{i_k}$ denotes the diffusion representation of $d_i$'s sub-diffusion ending at the position $k$ and $\tau$ is the hyper-parameter, known as the temperature. Inner-product is employed as the similarity function of InfoNCE.

For model inference, we formulate the predicted task as an information retrieval problem. Given an input diffusion, our model ranks a location list according to the possibility of being visited. And the possibility equals to the inner product of the diffusion representation and the location representations.

\subsection{Time Complexity}
The cost per training iteration of DGDI contains the computation of GCN (i.e., Eq.~\ref{eq:gcn}) and the self-attention layer (i.e., Eq.~\ref{eq:atn1} and Eq.~\ref{eq:atn2}). The graph convolution operations take $O(|\bm{E}|F^{2})$~\cite{DBLP:conf/iclr/KipfW17} where $|\bm{E}|$ denotes the number of edges and $F$ is dimension of latent embeddings. The time complexity of the self-attention layer is $O(|\bm{D}|K^{2}F)$ where $|\bm{D}|$ denote the size of observed cascades and $K$ is the length of the longest cascade. Thus, the overall complexity of DGDI is $O(|\bm{E}|F^{2} + |\bm{D}|K^{2}F)$. We compare DGDI with Inf-VAE~\cite{DBLP:conf/wsdm/SankarZK020} and one of our baselines SNIDSA~\cite{DBLP:conf/cikm/WangCL18}. The core operations of Inf-VAE are GCN with inner-product and a co-attention layer which take $O(|\bm{E}|F^{2} + |\bm{V}|^{2}F + |\bm{D}|K^{2}F)$. SNIDSA is a recurrent model with a structural attention network whose complexity is $O(|\bm{D}|KF^{2}+|\bm{V}|F^{2} + |\bm{E}|F)$. Since Inf-VAE requires the computation of each node pair, its time complexity is highest among these three models.

\begin{table}[t]
\small
\caption{Dataset statistics.} \label{table:dataset}
\setlength{\tabcolsep}{2.0mm}
\centering
\scalebox{0.95}{
\begin{tabular}{ccccc}
\toprule
\textbf{Dataset} & \textbf{\# Nodes} & \textbf{\# Links} & \textbf{\# Diffusions}\\
\midrule
\textbf{COVID-HK} & 3,274 & 603,667 & 2,536 \\
\textbf{COVID-MLC} & 4,091 & 58,079 & 1,887 \\
\bottomrule
\end{tabular}}
\end{table}

\section{Experiments}
\subsection{Experimental Settings}
\paragraph{\textbf{Datasets}} Two public available datasets are employed to evaluate our proposed model in COVID-19 transmission scenarios: COVID-HK and COVID-MLC. To better serve our problem, imported cases are excluded and we retain cases who have at least 2 locations. In terms of the geometric graph, two locations are connected if their distance is less than 3 kilometers.
\eat{A detailed dataset description can be found in Appendix.} \textbf{Note that all data are published by the government and do not contain sensitive information (i.e., anonymous).}
\begin{itemize}
    \item \textbf{COVID-HK.} The COVID-HK dataset is released by the Hong Kong government  \footnote{https://chp-dashboard.geodata.gov.hk/covid-19/en.html}, which records locations visited by the COVID-19 cases from Jan 28, 2020 to Sep 1, 2021. 
    \item \textbf{COVID-MLC.} The COVID-MLC dataset is provided by Beijing Advanced Innovation Center for Big Data and Brain Computing\footnote{https://github.com/BDBC-KG-NLP/COVID-19-tracker}, which records locations from Jan 1, 2020 to  March 22, 2020 and involves twelve provinces of Mainland, China.
\end{itemize}

\begin{table*}[t]
\caption{MAP@K and Recall@K comparison of different methods on two datasets: our model (denoted by bold) significantly outperforms the strongest baseline.}\label{table:overall}
\setlength{\tabcolsep}{2.0mm}
\centering
\scalebox{0.8}{
\begin{tabular}{@{}p{1.9cm}<{\centering}|p{1.1cm}<{\centering}p{1.1cm}<{\centering}p{1.1cm}<{\centering}p{1.1cm}<{\centering}p{1.1cm}<{\centering}p{1.4cm}<{\centering}|p{1.1cm}<{\centering}p{1.1cm}<{\centering}p{1.1cm}<{\centering}p{1.1cm}<{\centering}p{1.1cm}<{\centering}p{1.4cm}<{\centering}@{}}
\toprule
\textbf{Metric (\%)} & \multicolumn{6}{c|}{\textbf{Recall}} & \multicolumn{6}{c}{\textbf{MAP}} \\
\midrule
\multirow{2}{*}{\textbf{Model}}  & \multicolumn{3}{c}{\textbf{Covid-HK}} & \multicolumn{3}{c|}{\textbf{Covid-MLC}} & \multicolumn{3}{c}{\textbf{Covid-HK}} & \multicolumn{3}{c}{\textbf{Covid-MLC}} \\ 
& \textbf{@3} & \textbf{@5} & \textbf{@10}  & \textbf{@3} & \textbf{@5} & \textbf{@10}  & \textbf{@3} & \textbf{@5} & \textbf{@10} & \textbf{@3} & \textbf{@5} & \textbf{@10}\\
\midrule
\textbf{FMC} & 1.10$\pm$0.0 & 1.10$\pm$0.0 & 1.10$\pm$0.0 & 4.40$\pm$0.0 & 4.63$\pm$0.0 & 5.39$\pm$0.0 & 0.79$\pm$0.0 & 0.79$\pm$0.0 & 0.79$\pm$0.0 & 3.78$\pm$0.0 & 3.83$\pm$0.0 & 3.93$\pm$0.0 \\
\textbf{LSTM} & 0.73$\pm$0.2 & 0.88$\pm$0.3 & 1.18$\pm$0.2 & 2.73$\pm$0.2 & 3.24$\pm$0.2 & 4.03$\pm$0.2 & 0.56$\pm$0.2 & 0.59$\pm$0.2 & 0.63$\pm$0.2 & 2.17$\pm$0.2 & 2.26$\pm$0.2 &  2.37$\pm$0.1  \\
\textbf{DeepMove} & 0.94$\pm$0.2 & 1.15$\pm$0.2 & 1.39$\pm$0.2 & 3.29$\pm$0.2 & 3.64$\pm$0.2 & 4.55$\pm$0.2 & 0.77$\pm$0.2 & 0.82$\pm$0.1 & 0.85$\pm$0.1 & 2.82$\pm$0.2 & 2.90$\pm$0.2 & 3.02$\pm$0.2  \\
\textbf{GCN} & 1.77$\pm$0.2 & 2.17$\pm$0.3 & 3.15$\pm$0.3  & 4.20$\pm$0.5 & 4.83$\pm$0.5 & 6.17$\pm$0.6 &  1.19$\pm$0.2 & 1.28$\pm$0.2 & 1.40$\pm$0.2 & 3.10$\pm$0.3 & 3.25$\pm$0.3 & 3.42$\pm$0.2  \\
\textbf{GIN} & 1.92$\pm$0.2 & 2.46$\pm$0.4 & 3.42$\pm$0.9 & 4.95$\pm$0.2 & 6.28$\pm$0.4 & 8.15$\pm$0.3 & 1.26$\pm$0.2 & 1.39$\pm$0.2 & 1.51$\pm$0.2 & 4.12$\pm$0.2 & 4.42$\pm$0.2 & 4.67$\pm$0.2  \\
\midrule
\textbf{Topo-LSTM} & 1.78$\pm$0.2 & 2.25$\pm$0.3 & 3.30$\pm$0.4 & 4.66$\pm$0.1 & 5.58$\pm$0.2 & 7.17$\pm$0.1 &  1.28$\pm$0.1 & 1.40$\pm$0.1 & 1.54$\pm$0.1 & 3.17$\pm$0.2 & 3.38$\pm$0.2 & 3.58$\pm$0.2   \\
\textbf{SNIDSA} & 1.44$\pm$0.2 & 2.26$\pm$0.1 & 3.40$\pm$0.1 & 2.11$\pm$0.3 & 5.36$\pm$0.2 & 7.13$\pm$0.1 &  1.10$\pm$0.2 & 1.29$\pm$0.2 & 1.43$\pm$0.2  & 1.53$\pm$0.2 & 2.25$\pm$0.1 & 2.51$\pm$0.1   \\
\textbf{FOREST} & 1.77$\pm$0.3 & 2.53$\pm$0.4 & 3.50$\pm$0.3 & 4.55$\pm$0.2 & 5.32$\pm$0.4 & 6.48$\pm$0.4 &  1.14$\pm$0.1 & 1.32$\pm$0.2 & 1.45$\pm$0.2  & 3.88$\pm$0.3 & 4.06$\pm$0.3 & 4.21$\pm$0.3 \\
\textbf{Inf-VAE} & 2.03$\pm$0.2 & 2.54$\pm$0.2 & 3.73$\pm$0.2 & 6.39$\pm$0.2 & 7.33$\pm$0.1 & 8.64$\pm$0.3 & 1.56$\pm$0.1  & 1.66$\pm$0.1 & 1.80$\pm$0.1 & 5.02$\pm$0.2 & 5.20$\pm$0.2 & 5.44$\pm$0.2 \\
\midrule
\textbf{DGDI} & \textbf{3.76$\pm$0.3} & \textbf{4.89$\pm$0.4} & \textbf{6.95$\pm$0.9} & \textbf{7.04$\pm$0.3} & \textbf{8.88$\pm$1.1} & \textbf{12.19$\pm$1.1} & \textbf{2.74$\pm$0.2} & \textbf{2.99$\pm$0.2} & \textbf{3.25$\pm$0.2}  & \textbf{5.52$\pm$0.2} & \textbf{5.94$\pm$0.3} & \textbf{6.37$\pm$0.3}\\
\bottomrule
\end{tabular}}
\end{table*}

\paragraph{\textbf{Baselines}} We compared our proposed model with several methods. 
The first group focuses on diffusion modeling and contains FMC~\cite{DBLP:conf/gis/ZhangCL14}, LSTM~\cite{DBLP:journals/neco/HochreiterS97}, GRU~\cite{DBLP:conf/emnlp/ChoMGBBSB14} and DeepMove~\cite{DBLP:conf/www/FengLZSMGJ18}.
The second group focuses on geometric graph modeling and includes GCN~\cite{DBLP:conf/iclr/KipfW17} and GIN ~\cite{DBLP:conf/iclr/XuHLJ19}. It makes use of geometric graph information and represents diffusions by their last location representations.  The third group uses both geometric graph and diffusion information and includes Topo-LSTM~\cite{DBLP:conf/icdm/WangZLC17}, SNIDSA~\cite{DBLP:conf/cikm/WangCL18}, FOREST~\cite{DBLP:conf/ijcai/YangTSC019} and Inf-VAE~\cite{DBLP:conf/wsdm/SankarZK020}. These methods are the state-of-the-art in diffusion prediction tasks. 

\paragraph{Implementation details}
\eat{For GCN and GIN, we use the implementation of DGL\footnote{https://www.dgl.ai/}. }For Topo-LSTM\footnote{https://github.com/vwz/topolstm}, DeepMove\footnote{https://github.com/vonfeng/DeepMove}, SNIDSA\footnote{https://github.com/zhitao-wang/Sequential-Neural-Information-Diffusion-Model-with-Structure-Attention} and Inf-VAE\footnote{https://github.com/aravindsankar28/Inf-VAE}, we use the code provided by authors. We implement FOREST and our model by PyTorch. 

\paragraph{\textbf{Evaluation metrics}}
According to the timestamp, the dataset is split to 70\%, 10\% and 20\% for training, validation and testing. As stated in Section \ref{pdd}, we consider our prediction task as a ranking task. Specifically, the unvisited locations are ranked based on their predicted visit probabilities. Two widely used ranking metrics are adopted: Recall@K which denotes the fraction of infected locations among the top-k predicted locations, and MAP@K which jointly measures the existence and position of the target location in the rank list. 

\paragraph{\textbf{Setup}}
For fair comparison, all models are trained by Adam optimizer~\cite{DBLP:journals/corr/KingmaB14} with a learning rate of \{0.001, 0.0005, 0.0001\} and mini-batch size of 16. The dimension for representations and hidden states in all models are 32. For DGDI, we set $\lambda_2$ to 1 and tune $\lambda_1$ and $\tau$ within the ranges of \{0.1, 0.2, ..., 0.5\} and \{0.5, 0.6, ..., 1\}. We use GCN as our graph model and its layer number is searched within 2. For FOREST, we use recommended settings in their papers~\cite{DBLP:conf/ijcai/YangTSC019}. We run each method 5 times and report its average accuracy.

\subsection{\textbf{Overall Performance Comparison}}\label{sec:exp}
Table~\ref{table:overall} lists the performance for all methods where K is set to 3, 5, and 10. We can observe:
\begin{itemize}[leftmargin=*]
    \item Sequence Models that only utilize diffusions (i.e., FMC, LSTM, and DeepMove) perform worst in most cases, indicating that simply modeling the diffusion information is not sufficient. Although graph models (i.e., GCN and GIN) ignore previous visited locations to make predictions, they outperforms LSTM and DeepMove under all circumstances. One possibility is that most people only have a limited circle of activities (e.g., near their home or workplace) . Thus, integrating the geometric information allow the model to provide a better predictions.
    \item The integration of sequence models and the graph structure results in performance improvements. Topo-LSTM, SNIDSA, FOREST, Inf-VAE, and DGDI beat models based on pure sequence (i.e., FMC, LSTM, and DeepMove) in most cases, indicating the importance of jointly modeling the diffusion and graph information.
    \item DGDI significantly outperforms all the baselines on both MAP and Recall, which indicates that more to-be-visited locations are found and their ranks are higher in the predicted list. For example, our model beats the second best model Inf-VAE by 1.18 w.r.t. MAP@3 on COVID-HK and 0.93 w.r.t. MAP@10 on COVID-MLC. These improvements are owed to both the model architecture and the contrastive loss.
\end{itemize}

\begin{figure}[t]
\centering
\includegraphics[width=0.23\textwidth]{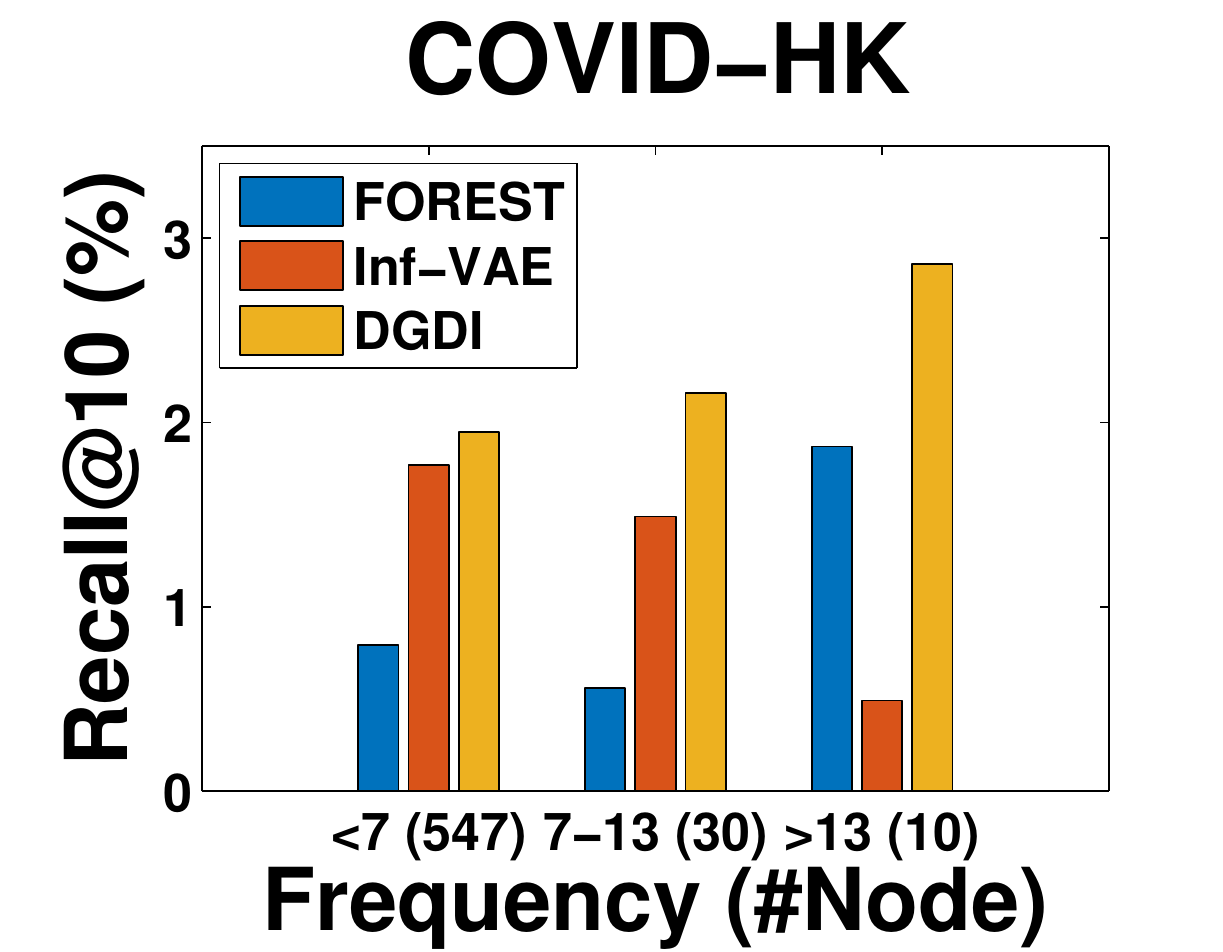}
\includegraphics[width=0.23\textwidth]{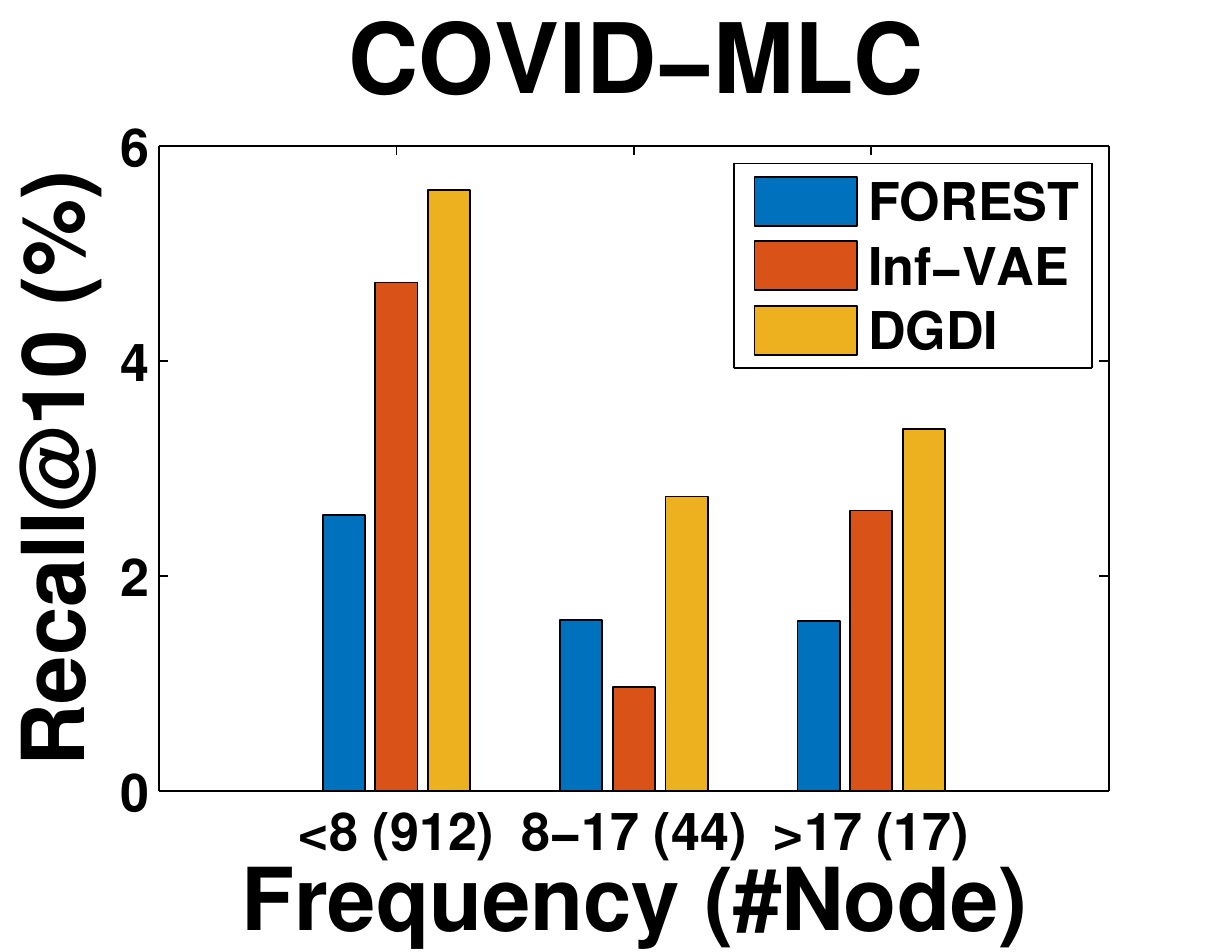} 
\caption{Recall@10 of COVID-HK and COVID-MLC on the frequency of location appearances. DGDI outperforms baselines on all groups.}\label{fig:long}
\end{figure}

\paragraph{\textbf{Comparison w.r.t. data sparsity}}
To investigate the model performance w.r.t. data sparsity of location appearance, we divide the frequency of location appearance into three groups which are $\{<7 (547), 7-13 (30), >13 (10)\}$ on COVID-HK and $\{<8 (912), 8-17 (44), >17 (17)\}$ on COVID-MLC, e.g., $<7 (547)$ indicates that there are 547 locations appears less than 7 times in training diffusions. We can see that most locations only have limited appearances and only a small number of locations appear frequently. The recall@10 of DGDI and the two strongest baselines (i.e., Inf-VAE and FOREST) are displayed in Figure~\ref{fig:long}. Specifically, the results of each group indicate how much they contribute to the overall results of the entire test set (i.e., the \textbf{sum} of these three results is equal to the results of the entire test set). As the figure shows, the model performs significantly better on head nodes than tail nodes since the performance of only a small number of head nodes is comparable to that of tail nodes. Moreover, we can observe that DGDI outperforms Inf-VAE and FOREST in all groups, demonstrating that DGDI is better than other models even under data sparsity.

\begin{table}[t]
\small
\caption{The MAP@10 and Recall@10 of ablation study \eat{on model design}. The results are average over 5 runs.} \label{table:abmap}
\setlength{\tabcolsep}{2.0mm}
\centering
\scalebox{0.9}{
\begin{tabular}{c|p{1.2cm}<{\centering}p{1.2cm}<{\centering}|p{1.3cm}<{\centering}p{1.2cm}<{\centering}}
\toprule
\multirow{2}{*}{\textbf{Method}} & \multicolumn{2}{c|}{\textbf{COVID-HK}} &  \multicolumn{2}{c}{\textbf{COVID-MLC}}  \\
& \textbf{Recall@10} & \textbf{MAP@10} & \textbf{Recall@10} & \textbf{MAP@10} \\
\midrule
\textbf{Remove GNN} & 2.28$\pm$0.3 & 1.18$\pm$0.1 & 8.79$\pm$0.4 & 4.58$\pm$0.1\\
$\bm{\lambda_1=0}$ & 6.56$\pm$0.5 & 2.75$\pm$0.3 & 11.12$\pm$1.1 & 5.82$\pm$0.5 \\
\textbf{Default} & 6.95$\pm$0.9 & 3.25$\pm$0.2 & 12.19$\pm$1.1 & 6.37$\pm$0.3\\
\bottomrule
\end{tabular}}
\end{table}

\subsection{\textbf{Model Ablation Study}}\label{sec:ab}
To investigate the contribution of our model design choices, the ablation study is conducted including:
\begin{enumerate}[leftmargin=*]
    \item \textbf{Remove GNN}: This method removes the GNN of the model design. Since $I(\bm{G};\bm{H})$ depends on GNN to obtain graph representations. If the GNN is removed, $I(\bm{G};\bm{H})$ is removed as well. Thus, this method only uses self-attention to derive diffusion representations.
    \item $\bm{\lambda_1 =0}$: Only use $\mathcal{L}_{U}$ to train the model and GNN is used to derive diffusion representations.
\end{enumerate}
 Table~\ref{table:abmap} shows MAP@10 and Recall@10 of our model and its variants. The last row denotes the performance of our model in default settings.
 
 Due to the lack of graph information, the performances decrease on both datasets by a large margin after removing the GNN, which reflects the importance of geometric information when predicting COVID-19 diffusion. Maximizing $\mathcal{L}_G$ enforces diffusion embeddings to encode the geometric information. When setting $\lambda_1 =0$, the performance drops in all cases, which demonstrates the effectiveness of $\mathcal{L}_G$. Specifically, the relative improvement w.r.t. MAP@10 on COVID-HK and COVID-MLC is 18.18\% and 9.45\%, demonstrating the effectiveness of DGDI.

\eat{
\subsection{Case Study}
To better understand the effect of contrastive loss $\mathcal{L}_G$, a case study is conducted where the top-ranking locations of our model and its ablation model (i.e., $\bm{\lambda_1 = 0}$) are plotted in Figure~\ref{fig:case}. We use black points A and B to denote the observed visited locations and the goal is to predict the next location C. The red and blue points are the prediction of our model and the ablation model. We can observe that although both models successfully predict the next location, all the predicted locations of our model are near the visited points, while only three of the ablation model's predictions are close. The reason is that maximizing $\mathcal{L}_G$ forces the model to search places that are geometrically near the previously visited locations.

\begin{figure}[t]
\centering
\includegraphics[width=0.48\textwidth]{CaseStudy.pdf}
\caption{The predicted location distribution on the map (micro view). Our model can provide a more precise region containing the target location while predictions of the ablation model are more diverse and random.}\label{fig:case}
\end{figure}}

\section{Conclusion}
In this work, we argue that human mobility modeling plays an important role during the NPI design and propose to predict case locations in the COVID-19 transmission. We present a novel mutual information maximization framework, named DGDI, to jointly learn the high-quality location, geometric graph, and diffusion representations. A lower bound of DGDI consists of the sum of two univariate mutual information is derived to optimize the model. Compared with other related works, DGDI can better handle data sparsity of diffusion and location appearance. Two COVID-19 datasets are proposed to facilitate the research of COVID-19 mobility modeling. Experimental results on these two real-world datasets show that DGDI outperforms other competitors by a significant margin on both Recall and MAP. 

\section*{Acknowledgements}
The research of Li was supported by NSFC Grant No. 62206067, Tencent AI Lab Rhino-Bird Focused Research Program RBFR2022008 and Guangzhou-HKUST(GZ) Joint Funding Scheme. The research of Tsung was supported by the Hong Kong RGC General Research Funds 16216119 and Foshan HKUST Projects FSUST20-FYTRI03B.
\bibliography{aaai23}

\end{document}